\newif\ifisarxiv
\isarxivtrue

\ifisarxiv
\documentclass[11pt]{article}
\usepackage{fullpage}
\else
\documentclass{article}
\PassOptionsToPackage{numbers,compress,sort}{natbib}
\usepackage{sty/neurips2019/neurips_2019}
\fi

\usepackage[utf8]{inputenc} % allow utf-8 input
\usepackage[T1]{fontenc}    % use 8-bit T1 fonts
\usepackage{hyperref}       % hyperlinks
\usepackage{url}            % simple URL typesetting
\usepackage{booktabs}       % professional-quality tables
\usepackage{amsmath}
\usepackage{color}
\usepackage{graphicx}
\usepackage{multirow}
\usepackage{amsfonts}       % blackboard math symbols
\usepackage{nicefrac}       % compact symbols for 1/2, etc.
\usepackage{microtype}      % microtypography
\usepackage{wrapfig}
\usepackage{algorithm}
\usepackage{algorithmic}
\usepackage{xfrac}
\usepackage{caption}

\def\d{\mathrm{d}}

\def\simiid{\overset{\textnormal{\fontsize{6}{6}\selectfont
i.i.d.}}{\sim}}

\DeclareMathOperator{\adj}{\mathrm{adj}}

\def\Sigmab{\mathbf{\Sigma}}
\def\Sigmabh{\widehat{\Sigmab}}

\def\val {{\mathrm{val}}}

\def\g {\mathbf{g}}

\def\ee{\mathrm{e}}

\def\Lc{\mathcal{L}}
\def\Lch{\widehat{\Lc}}

\def\p{\mathbf p}

\def\Y{\mathbf Y}
\def\R{\mathbf R}
\def\H{\mathbf H}
\def\Hbh{\widehat{\H}}

\def\Zbt{\widetilde{\Z}}

\ifx\BlackBox\undefined
\newcommand{\BlackBox}{\rule{1.5ex}{1.5ex}}  % end of proof
\fi

\DeclareMathOperator*{\argmin}{\mathop{\mathrm{argmin}}}

\def\x{\mathbf x}

\def\w{\mathbf w}
\def\v{\mathbf v}

\def\wbt{\widetilde{\mathbf w}}
\def\p{\mathbf{p}}
\def\pbh{\widehat{\p}}
\def\e{\mathbf e}
\def\zero{\mathbf 0}
\def\one{\mathbf 1}
\def\u{\mathbf u}

\def\X{\mathbf X}

\def\B{\mathbf B}
\def\A{\mathbf A}

\def\M{\mathbf M}

\def\Z{\mathbf Z}

\def\Zbt{\widetilde{\mathbf Z}}
\def\I{\mathbf I}

\def\A{\mathbf A}

\def\E{\mathbb E}

\def\R{\mathbb R} 
\def\tr{\mathrm{tr}}

\newcommand{\defeq}{\stackrel{\textit{\tiny{def}}}{=}}

\let\origtop\top
\renewcommand\top{{\scriptscriptstyle{\origtop}}} % this makes transpose not so big

\definecolor{silver}{cmyk}{0,0,0,0.3}
\definecolor{yellow}{cmyk}{0,0,0.9,0.0}
\definecolor{reddishyellow}{cmyk}{0,0.22,1.0,0.0}
\definecolor{black}{cmyk}{0,0,0.0,1.0}
\definecolor{darkYellow}{cmyk}{0.2,0.4,1.0,0}
\definecolor{darkSilver}{cmyk}{0,0,0,0.1}
\definecolor{grey}{cmyk}{0,0,0,0.5}
\definecolor{darkgreen}{cmyk}{0.6,0,0.8,0}

\newcommand{\Green}[1]{{\color{darkgreen}  {#1}}}
\newcommand{\Blue}[1]{\color{blue}{#1}\color{black}}
\newcommand{\Brown}[1]{{\color{brown}{#1}\color{black}}}

\newenvironment{proofof}[2]{\par\vspace{2mm}\noindent\textbf{Proof of {#1} {#2}}\ }{\hfill\BlackBox}

\ifx\proof\undefined
\newenvironment{proof}{\par\noindent{\bf Proof\ }}{\hfill\BlackBox\\[2mm]}
\fi

\ifx\theorem\undefined
\newtheorem{theorem}{Theorem}
\fi

\ifx\example\undefined
\newtheorem{example}{Example}
\fi

\ifx\property\undefined
\newtheorem{property}{Property}
\fi

\ifx\lemma\undefined
\newtheorem{lemma}[theorem]{Lemma}
\fi

\ifx\proposition\undefined
\newtheorem{proposition}[theorem]{Proposition}
\fi

\ifx\remark\undefined
\newtheorem{remark}[theorem]{Remark}
\fi

\ifx\corollary\undefined
\newtheorem{corollary}[theorem]{Corollary}
\fi

\ifx\definition\undefined
\newtheorem{definition}{Definition}
\fi

\ifx\conjecture\undefined
\newtheorem{conjecture}[theorem]{Conjecture}
\fi

\ifx\axiom\undefined
\newtheorem{axiom}[theorem]{Axiom}
\fi

\ifx\claim\undefined
\newtheorem{claim}[theorem]{Claim}
\fi

\ifx\assumption\undefined
\newtheorem{assumption}[theorem]{Assumption}
\fi

\title{Distributed estimation of the inverse Hessian by \\
  determinantal averaging}
\ifisarxiv\date{}\def\And{\and}\fi
\author{
      Micha{\l } Derezi\'{n}ski \\
Department of Statistics\\
University of California, Berkeley\\
\texttt{mderezin@berkeley.edu}\\
\And
Michael W. Mahoney\\
ICSI and Department of Statistics\\
University of California, Berkeley\\
\texttt{mmahoney@stat.berkeley.edu}
}

\begin{document}

\maketitle

\begin{abstract}
In distributed optimization and distributed numerical linear algebra,
we often encounter an \emph{inversion bias}: if we want to compute a
quantity that depends on the inverse of a sum of distributed matrices,
then the sum of the inverses does not equal the inverse of the sum.  
An example of this occurs in distributed Newton's method, where we
wish to compute (or implicitly work with) the inverse Hessian
multiplied by the gradient.  
In this case, locally computed estimates are biased, and so taking a
uniform average will not recover the correct solution.  
To address this, we propose \emph{determinantal averaging}, a new
approach for correcting the inversion bias. 
This approach involves reweighting the local estimates of the Newton's
step proportionally to the determinant of the local Hessian estimate,
and then averaging them together to obtain an improved global
estimate.  
This method provides the first known distributed Newton step that is
\emph{asymptotically consistent}, i.e., it recovers the exact step in
the limit as the number of distributed partitions grows to infinity.  
To show this, we develop new expectation identities and moment bounds
for the determinant and adjugate of a random matrix.  
Determinantal averaging can be applied not only to Newton's method,
but to computing any quantity that is a linear tranformation of a
matrix inverse, e.g., taking a trace of the inverse covariance matrix,
which is used in data uncertainty quantification.  
\end{abstract}

\section{Introduction}
\label{s:intro}

Many problems in machine learning and optimization require that we produce an accurate estimate of a square matrix $\H$ (such as the Hessian of a loss function or a sample covariance), while having access to many copies of some unbiased estimator of $\H$, i.e., a random matrix $\Hbh$ such that $\E[\Hbh]=\H$. 
In these cases, taking a uniform average of those independent copies provides a natural strategy for boosting the estimation accuracy, essentially by making use of the law of large numbers: $\frac1m\sum_{t=1}^m\Hbh_t\rightarrow
\H$. 
For many other problems, however, we are more interested in the inverse (Hessian/covariance) matrix $\H^{-1}$, and it is necessary or desirable to work with $\Hbh^{-1}$ as the estimator. 
Here, a na\"{\i}ve averaging approach has certain fundamental limitations (described in more detail below).
The basic reason for this is that $\E[\Hbh^{-1}]\neq \H^{-1}$, i.e., that there is what may be called an \emph{inversion bias}.

In this paper, we propose a method to address this inversion bias challenge.
The method uses a \emph{weighted} average, where the weights are carefully chosen to compensate for and correct the bias. 
Our motivation comes from distributed Newton's method (explained shortly), where combining independent estimates of the inverse Hessian is desired, but our method is more generally applicable, and so we first state our key ideas in a more general context.
\begin{theorem}\label{t:key}
  Let $s_i$ be independent random variables and $\Z_i$ be fixed square rank-$1$ matrices.
  If \,$\Hbh = \sum_i s_i\Z_i$ is invertible almost surely, then the inverse of the matrix $\H=\E[\Hbh]$ can be expressed~as:
\begin{align*}
\H^{-1} =
  \frac{\E\big[\!\det(\Hbh)\Hbh^{-1}\big]}{\E\big[\!\det(\Hbh)\big]}.
\end{align*}
\end{theorem}
To demonstrate the implications of Theorem \ref{t:key}, suppose that our goal is to estimate $F(\H^{-1})$ for some linear function $F$. 
For example, in the case of Newton's method $F(\H^{-1}) = \H^{-1}\g$, where $\g$ is the gradient and $\H$ is the Hessian. 
Another example would be $F(\H^{-1})=\tr(\H^{-1})$, where $\H$ is the
covariance matrix of a dataset and $\tr(\cdot)$ is the matrix trace, which is useful for uncertainty quantification.
For these and other cases, consider the following estimation of $F(\H^{-1})$, which takes
an average of the individual estimates $F(\Hbh_t^{-1})$, each weighted by the
determinant of $\Hbh_t$, i.e.,
\begin{align*}
%\textit{determinantal averaging:}\quad
\textbf{Determinantal Averaging:}\quad
  \hat{F}_m=
  \frac{\sum_{t=1}^ma_tF(\Hbh_t^{-1})}{\sum_{t=1}^ma_t},
  \qquad a_t = \det(\Hbh_t)  .
\end{align*}
By applying the law of large  numbers (separately to the numerator and the denominator), 
Theorem~\ref{t:key} easily implies that if $\Hbh_1,\dots,\Hbh_m$ are i.i.d.~copies
of $\Hbh$ then this 
\emph{determinantal averaging estimator}
is asymptotically consistent, i.e., $\hat{F}_m\rightarrow F(\H^{-1})$, almost surely. 
This 
determinantal averaging
estimator
is particularly useful when problem constraints do not allow us to compute $F\big((\frac1m\sum_t\Hbh_t)^{-1}\big)$, e.g., when the matrices are distributed and not easily combined. 
% While our theory assumes that both $\det(\Hbh_t)$ and $F(\Hbh_t^{-1})$
% are computed exactly, we conjecture that our approach can be extended
% to using approximate calculations.

To establish finite sample convergence guarantees for estimators
obtained via determinantal averaging, we establish the following
matrix concentration result. 
We state it separately since it is technically interesting and since
its proof requires novel bounds for the higher moments of the
determinant of a random matrix, which is likely to be of independent
interest. Below and throughout the paper, $C$ denotes an absolute
constant and ``$\preceq$'' is the L\"owner order on positive semi-definite (psd) matrices.
\begin{theorem}\label{t:finite}
Let $\Hbh=\frac1k\sum_{i=1}^nb_i\Z_i\,+\B$ and $\H=\E[\Hbh]$, where
$\B$ is a positive definite $d\times d$ matrix and 
$b_i$ are i.i.d.~$\mathrm{Bernoulli}(\frac kn)$.
Moreover, assume that all $\Z_i$ are psd, $d\times d$ and
rank-$1$. If
$k\geq C\frac{\mu d^2}{\eta^2}\log^3\!\frac d\delta$\, for\,
$\eta\in(0,1)$\, and\, $\mu=\max_i\|\Z_i\H^{-1}\|/d$, then 
\begin{align*}
  \Big(1-\frac\eta {\sqrt{m}}\Big)\cdot\H^{-1} \preceq
  \frac{\sum_{t=1}^ma_t\Hbh_t^{-1}}{\sum_{t=1}^ma_t}\preceq 
  \Big(1+\frac\eta {\sqrt{m}}\Big)\cdot\H^{-1}\quad \text{with
  probability }\geq 1-\delta,
\end{align*}
where $\Hbh_1,\dots,\Hbh_m\simiid\Hbh$ and $a_t=\det(\Hbh_t)$.
\end{theorem}

\subsection{Distributed Newton's method}\label{ss:newton}
To illustrate how determinantal averaging can be useful in the context
of distributed optimization, consider the task of batch minimization
of a convex loss over vectors $\w\in\R^d$, defined as follows:
\vspace{-2mm}
\begin{align}
  \Lc(\w) \defeq \frac1n \sum_{i=1}^n\ell_i(\w^\top\x_i)\
  +\frac\lambda2\|\w\|^2,\label{eq:loss}
\end{align}
where $\lambda>0$, and $\ell_i$ are convex, twice differentiable and smooth.
Given a vector $\w$, Newton's method dictates that the correct way to move towards the optimum is to perform an update $\wbt = \w - \p$, with $\p= \nabla^{-2}\!\Lc(\w)\,\nabla\!\Lc(\w)$, where $\nabla^{-2}\!\Lc(\w)=(\nabla^2\Lc(\w))^{-1}$ denotes the inverse Hessian of $\Lc$ at $\w$.%
\footnote{Clearly, one would not actually compute the inverse of the Hessian explicitly~\cite{XRM17_theory_TR,YXRM18_TR}. We describe it this way for simplicity. Our results hold whether or not the inverse operator is computed explicitly.} 
Here, the Hessian and gradient are:
\begin{align*}
  % \nabla^2_{\!\w}\Lc
  \nabla^2\!\Lc(\w)
  = \frac1n\sum_i \ell_i''(\w^\top\x_i)\,
  \x_i\x_i^\top\, +\lambda\I,\quad\text{and}\quad\nabla\!\Lc(\w) = \frac1n\sum_i \ell_i'(\w^\top\x_i)\,\x_i\ +
        \lambda\w.
\end{align*}
For our distributed Newton application, we study a distributed computational model, where a single machine has access to a subsampled version of $\Lc$ with sample size parameter $k\ll n$:
\begin{align}
  \Lch(\w)\defeq \frac1k\sum_{i=1}^nb_i\ell_i(\w^\top\x_i)\
  +\frac\lambda2\|\w\|^2,\quad\text{where}\quad b_i\sim\mathrm{Bernoulli}\big(k/n\big).\label{eq:l-est}
\end{align}
Note that $\Lch$ accesses on average $k$ loss components $\ell_i$ ($k$
is the expected local sample size), and 
moreover, $\E\big[\Lch(\w)\big] = \Lc(\w)$ for any $\w$. The goal is
to compute local estimates of the Newton's step $\p$ in a
communication-efficient manner (i.e., by only sending $O(d)$
parameters from/to a single machine), then combine them into a
better global estimate. The gradient has size $O(d)$ so it can be computed exactly within
this communication budget (e.g., via map-reduce), however the Hessian has
to be approximated locally by each machine. Note that other computational models
can be considered, such as those where the global gradient is not
computed (and local gradients are used instead).
 \begin{wrapfigure}{r}{0.45\textwidth}
\vspace{-.8cm}
\begin{center}
  \includegraphics[width=.47\textwidth]{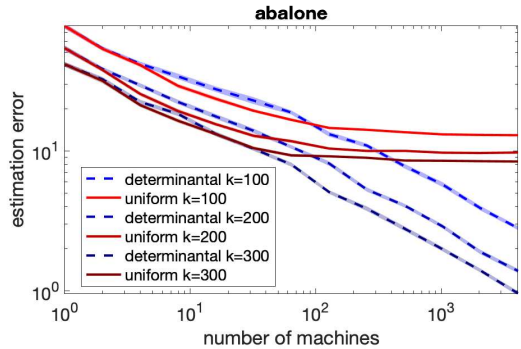}
  \vspace{-6mm}
  \captionof{figure}{Newton step estimation error versus number of machines,
     averaged over 100 runs (shading is
     standard error) for a libsvm dataset \cite{libsvm}. More plots in
     Appendix~\ref{a:experiments}.} 
\label{fig:abalone}
\end{center}
\vspace{-.8cm}
\end{wrapfigure}

Under the constraints described above, the most natural strategy is to use directly the Hessian of the locally subsampled loss $\Lch$ (see, e.g., GIANT~\cite{distributed-newton}), resulting in the approximate Newton step $\pbh
  =\nabla^{-2}\!\Lch(\w)\,\nabla\!\Lc(\w)$.
Suppose that we independently construct $m$ i.i.d.~copies of this estimate:
$\pbh_1,\dots,\pbh_m$ (here, $m$ is the number of machines). 
Then, for sufficiently large $m$, taking a simple average of the estimates will stop converging to $\p$ because of the inversion bias: $ \frac1m\sum_{t=1}^m\pbh_t\rightarrow \E\big[\pbh\big] \neq \p$. 
Figure \ref{fig:abalone} shows this by plotting the estimation error (in Euclidean distance) of the averaged Newton step estimators, when the weights are uniform and determinantal (for more details and plots, see Appendix~\ref{a:experiments}). 

The only way to reduce the estimation error beyond a certain point is to increase the local sample size $k$ (thereby reducing the inversion bias), which raises the computational cost per machine. 
%On the other hand, 
Determinantal averaging corrects the inversion bias so that estimation error can always be decreased by adding more machines without increasing the local sample size. 
From the preceding discussion we can easily show that determinantal
averaging leads to an asymptotically consistent estimator. 
This is a corollary of Theorem~\ref{t:key}, as proven in
Section~\ref{sxn:expectation_identities}. 

\begin{corollary}\label{t:limit}
Let $\{\Lch_t\}_{t=1}^\infty$ be i.i.d. samples of \eqref{eq:l-est}
and define $a_t =\det\!\big(\nabla^2\!\Lch_t(\w)\big)$. Then:
\begin{align*}
\frac{\sum_{t=1}^ma_t\,\pbh_t
  %\big(\nabla^2\!\Lch(\w)\,\big)^{-1}\nabla_{\!\w}\Lc
  }{\sum_{t=1}^ma_t}\
  \overset{\text{a.s.}}{\underset{m\rightarrow\infty}{\longrightarrow}}\ 
\p, 
\quad\text{where}\quad
\pbh_t=\nabla^{-2}\!\Lch_t(\w)\,\nabla\!\Lc(\w)
\ \text{ and }\
  \p=\nabla^{-2}\!\Lc(\w)\,\nabla\!\Lc(\w).
  \end{align*}
\end{corollary}
The (unnormalized) determinantal weights can be computed locally in
the same time as it takes to compute the Newton estimates so they do
not add to the overall cost of the procedure. While this
result is only an asymptotic statement, it holds with virtually no assumptions
on the loss function (other than twice-differentiability) or the
expected local sample size $k$. 
However, with some additional assumptions we will now establish a
convergence guarantee with a finite number of machines $m$ by bounding
the estimation error for the determinantal averaging estimator of the Newton step.

In the next result, we use Mahalanobis distance, denoted $\|\v\|_{\M}=\sqrt{\v^\top\M\v}$, to measure the error of the Newton step estimate (i.e., the deviation from optimum $\p$), with $\M$ chosen as the Hessian of $\Lc$. 
This choice is motivated by standard convergence analysis of Newton's method, discussed next. 
This is a corollary of Theorem \ref{t:finite}, as explained in Section \ref{s:error}.
\begin{corollary}\label{t:error}
For any $\delta,\eta\!\in\!(0,1)$  if expected local sample size satisfies
$k\geq C\eta^{-2}\mu d^2\log^3\!\frac d\delta$ 
then
\begin{align*}
  \bigg\|\,\frac{\sum_{t=1}^ma_t\,\pbh_t}{\sum_{t=1}^ma_t} \,-\,
  \p\,\bigg\|_{\nabla^2\!\Lc(\w)}\!\!\leq \frac\eta
  {\sqrt{m}}\cdot\big\|\,\p\,\big\|_{\nabla^2\!\Lc(\w)}\quad
  \text{with probability }\geq 1-\delta,
\end{align*}
where $\mu= \frac1d\max_i\ell_i''(\w^\top\!\x_i)\|\x_i\|^2_{\nabla^{-2}\!\Lc(\w)}$, and $a_t$, $\pbh_t$ and $\p$ are defined as in
Corollary \ref{t:limit}.   
\end{corollary}

We next establish how this error bound impacts the convergence
guarantees offered by Newton's method. Note that under our assumptions
$\Lc$ is strongly convex so there is a unique minimizer
$\w^*=\argmin_\w\Lc(\w)$. We ask how the distance from
optimum, $\|\w-\w^*\|$, changes after we make an update
$\wbt=\w-\pbh$. For this, we have to assume that the Hessian
matrix is $L$-Lipschitz as a function of $\w$. After this standard
assumption, a classical analysis of the Newton's method reveals that
Corollary~\ref{t:error} implies the following Corollary~\ref{c:rate} (proof in
Appendix~\ref{a:newton}).
\begin{assumption}\label{a:lipschitz}
  The Hessian is $L$-Lipschitz: 
  $\|\nabla^2\!\Lc(\w)-\nabla^2\!\Lc(\wbt)\|\leq L\,\|\w-\wbt\|$
  for any $\w,\wbt\in\R^d$.
\end{assumption}
\begin{corollary}\label{c:rate}
For any $\delta,\eta\!\in\!(0,1)$  if expected local sample size satisfies
$k\geq C\eta^{-2}\mu d^2\log^3\!\frac d\delta$ then under
Assumption \ref{a:lipschitz} it holds with probability at least
$1-\delta$ that
\begin{align*}
  \big\|\wbt-\w^*\big\|\leq \max\Big\{\frac{\eta }{\sqrt{m}}\sqrt{\kappa}\,
  \big\|\w-\w^*\big\|,\ \ \frac{2L}{\lambda_{\min}} \, \big\|\w-\w^*\big\|^2\Big\} \quad
  \text{ for }\
  \wbt=\w-\frac{\sum_{t=1}^ma_t\,\pbh_t}{\sum_{t=1}^ma_t},
\end{align*}
where $C$, $\mu$, $a_t$ and $\pbh_t$ are defined as in Corollaries
\ref{t:limit} and \ref{t:error}, while $\kappa$ and $\lambda_{\min}$ are the condition
number and smallest eigenvalue of $\nabla^2\!\Lc(\w)$, respectively.
\end{corollary}
The bound is a maximum of a linear and a quadratic convergence term. As
$m$ goes to infinity and/or $\eta$ goes to $0$ the approximation
coefficient $\alpha=\frac\eta{\sqrt{m}}$ in the linear term disappears and we obtain exact
Newton's method, which exhibits quadratic convergence (at least locally around
$\w^*$). However, decreasing $\eta$ 
means increasing $k$ and with it the average computational cost per
machine. Thus, to preserve the
quadratic convergence while maintaining a computational budget per
machine, as the optimization progresses we have to
increase the number of machines $m$ while keeping $k$ fixed. This is
only possible when we correct for the inversion bias, which is done
by determinantal averaging.
\subsection{Distributed data uncertainty quantification}

Here, we consider another example of when computing a compressed
linear representation of the inverse matrix is important. Let $\X$ be
an $n\times d$ matrix where the rows $\x_i^\top$ represent samples
drawn from a population for statistical analysis. The sample
covariance matrix $\Sigmab=\frac1n\X^\top\X$ holds the information
about the relations between the features. Assuming that $\Sigmab$ is invertible,
the matrix $\Sigmab^{-1}$, also known as the precision matrix, is
often used to establish a degree of confidence we have in the data
collection \cite{uncertainty1}. The diagonal
elements of $\Sigmab^{-1}$ are particularly useful since they hold the
variance information of each individual feature. Thus, efficiently
estimating either the entire diagonal, its trace, or some subset of its
entries is of practical interest
\cite{inverse-portfolio,estimating-diagonal,high-performance-uncertainty}.
We consider the distributed setting where data is separately
stored in batches and each local covariance is modeled as:
\vspace{-4mm}
\begin{align*}
  \Sigmabh =
  \frac1k\sum_{i=1}^nb_i\x_i\x_i^\top,\quad\text{where}\quad
  b_i\sim\mathrm{Bernoulli}(k/n).
\end{align*}
For each of the local covariances $\Sigmabh_1,\dots,\Sigmabh_m$, we
compute its compressed uncertainty information:
$F\big((\Sigmabh_t+\frac\eta{\sqrt{m}}\I)^{-1}\big)$, 
where we added a small amount of ridge to ensure invertibility%
\footnote{Since the ridge term vanishes as $m$ goes to
  infinity, we are still estimating the ridge-free quantity
  $F(\Sigmab^{-1})$.}. Here,
$F(\cdot)$ may for example denote the trace or the vector of diagonal
entries. We arrive at the following asymptotically consistent
estimator for $F(\Sigmab^{-1})$:
\begin{align*}
  \hat F_m = \frac{\sum_{t=1}^ma_{t,m}
  F\big((\Sigmabh_t+\frac\eta{\sqrt{m}}\I)^{-1}\big)}
  {\sum_{t=1}^m a_{t,m}},\quad\mbox{where}\quad a_{t,m}= \det\!\big(\Sigmabh_t+\tfrac\eta{\sqrt{m}}\I\big).
\end{align*}
Note that the ridge term $\frac\eta{\sqrt{m}}\I$ decays to zero as $m$
goes to infinity, which is why $\hat F_m\rightarrow
F(\Sigmab^{-1})$. Even though this limit holds for any local sample
size $k$, in practice we should choose $k$ 
sufficiently large so that $\Sigmabh$ is well-conditioned.
In particular, Theorem \ref{t:finite} implies that if $k\geq
2C\eta^{-2}\mu d^2\log^3\!\frac d\delta$, where
$\mu=\frac1d\max_i\|\x_i\|_{\Sigmab^{-1}}^2$, then for
$F(\cdot)=\tr(\cdot)$ we have  
% $\hat  F_m/\tr(\Sigmab^{-1})\in[-\frac\eta{\sqrt{m}},
% \frac\eta{\sqrt{m}}]$ w.p.~$1-\delta$.
$|\hat F_m - \tr(\Sigmab^{-1})|\leq \frac\eta {\sqrt{m}}\cdot
 \tr(\Sigmab^{-1})$ w.p.~$1-\delta$.
% For such $k$, a finite sample guarantee follows from
% Theorem~\ref{t:finite}.
%As an example, below $F(\cdot)$ is the matrix trace $\tr(\cdot)$.
% \begin{corollary}\label{t:trace}
%   %For any $\delta,\eta\!\in\!(0,1)$,
% If $k\geq C\frac{\mu d^2}{\eta^2}\log^3\!\frac d\delta$, where $\mu= \frac1d\max_i\|\x_i\|^2$,
% then $\big|\hat F_m -
%   \tr(\Sigmab^{-1})\big|\leq \frac\eta
%   {\sqrt{m}}\cdot \tr(\Sigmab^{-1})$ %w.p.~$1-\delta.$
% % \begin{align*}
% %   \big|\hat F_m -
% %   \tr(\Sigmab^{-1})\big|\leq \frac\eta
% %   {\sqrt{m}}\cdot \tr(\Sigmab^{-1})\quad
% %   \text{with probability }\geq 1-\delta.
% % \end{align*}
% \end{corollary}

\subsection{Related work}
Many works have considered averaging strategies for combining
distributed estimates, particularly in the context of statistical learning
and optimization. This research is particularly important in
\emph{federated learning}
\cite{konecny-federated16a,konecny-federated16b}, where data are spread
out accross a large network of devices with small local storage and
severely constrained communication bandwidth. Using averaging to
combine local estimates has been studied in a number of learning
settings \cite{mcdonald09,mcdonald10} as well as for first-order
stochastic optimization \cite{parallel-sgd,agarwal-duchi11}.
For example, \cite{zhang-duchi-wainwright13} examine the effectiveness of simple
uniform averaging of emprical risk minimizers and also propose a
bootstrapping technique to reduce the bias.

More recently, distributed averaging methods gained considerable
attention in the context of second-order optimization, where the
Hessian inversion bias is of direct concern. \cite{dane}
propose a distributed approximate Newton-type method (DANE) which
under certain assumptions exhibits low bias. This was later extended
and improved upon by \cite{disco,aide}. The GIANT method of
\cite{distributed-newton} most closely follows our setup from Section
\ref{ss:newton}, providing non-trivial guarantees for uniform
averaging of the Newton step estimates $\pbh_t$ (except they use
with-replacement uniform sampling, whereas we use without-replacement,
but that is typically a negligible difference). A related analysis of
this approach is provided in the context of ridge regression by
\cite{sketched-ridge-regression}. Finally, \cite{naman17} propose a
different estimate of the Hessian inverse by 
using the Neuman series expansion, and use it to construct Newton step
estimates which exhibit low bias when the Hessian is sufficiently
well-conditioned.

Our approach is related to recent developments in determinantal
subsampling techniques (e.g.~volume sampling), which have been shown
to correct the inversion bias in the context of least squares
regression \cite{unbiased-estimates,correcting-bias}. However, despite
recent progress
\cite{unbiased-estimates-journal,leveraged-volume-sampling},
volume sampling is still far too computationally expensive to be
feasible for distributed optimization. Indeed, often uniform
sampling is the only practical choice in this context.

With the exception of the expensive volume sampling-based methods, all
of the approaches discussed above, even under
favorable conditions, use \emph{biased} estimates of the desired
solution (e.g., the exact Newton step). Thus, when the number of
machines grows sufficiently large, with fixed local sample size, the
averaging no longer provides any improvement. This is in contrast to
our determinantal averaging, which converges \emph{exactly} to the
desired solution and requires no expensive subsampling.
Therefore, it can scale with an arbitrarily large number of machines. 

\ifisarxiv\vspace{-2mm}\fi
\section{Expectation identities for determinants and adjugates}
\label{sxn:expectation_identities}

In this section, we prove Theorem \ref{t:key} and Corollary \ref{t:limit}, establishing that determinantal averaging is asymptotically consistent. 
To achieve this, we establish a lemma involving two expectation identities.

For a square $n\times n$ matrix $\A$, we use $\adj(\A)$ to denote its adjugate, defined as an $n\times n$ matrix whose $(i,j)$th entry is $(-1)^{i+j}\det(\A_{-j,-i})$, where $\A_{-j,-i}$ denotes $\A$ without $j$th row and $i$th column. 
The adjugate matrix provides a key connection between the inverse and the determinant because for any invertible matrix $\A$, we have $\adj(\A) = \det(\A)\A^{-1}$. 
In the following lemma, we will also use a formula called Sylvester's theorem, relating the adjugate and the determinant: $\det(\A+\u\v^\top)=\det(\A) + \v^\top\!\adj(\A)\u$.
\begin{lemma}\label{t:det}
  For $\A = \sum_i s_i\Z_i$, where $s_i$ are independently random  and $\Z_i$
  are square and rank-$1$,
\begin{align*}
\mathrm{(a)}\ \    \E\big[\det(\A)\big] =
  \det\!\big(\E[\A]\big)\quad\text{and}\quad\mathrm{(b)}\ \    \E\big[\adj(\A)\big] = \adj\!\big(\E[\A]\big).
\end{align*}
\end{lemma}
\begin{proof}
We use induction over the number of components in the sum. If there is
only one component, i.e., $\A=s\Z$, then $\det(\A)=0$ a.s.~unless $\Z$
is $1\!\times\!1$, in which case  (a) is trivial, and $(b)$ follows
similarly.  Now, suppose we showed the hypothesis when the number of 
components is $n$ and let $\A=\sum_{i=1}^{n+1}s_i\Z_i$. %  In the
% following derivation we use a variant of Sylvester's theorem which
% states that $\det(\M+\u\v^\top)=\det(\M)+\v^\top\!\adj(\M)\u$.
Setting $\Z_{n+1}=\u\v^\top$, we have:
\vspace{-2mm}
\begin{align*}
  \E\big[\det(\A)\big]
  &= \E\bigg[\det\!\Big(\sum_{i=1}^ns_i\Z_i+s_{n+1}\u\v^\top\Big)\bigg]
\\\text{(Sylvester's Theorem)}\quad &=\E\bigg[\det\!\Big(\sum_{i=1}^ns_i\Z_i\Big) +
     s_{n+1}\v^\top\!\adj\!\Big(\sum_{i=1}^ns_i\Z_i\Big)\u\bigg]
  \\ \text{(inductive hypothesis)}\quad &=\det\!\bigg(\E\Big[\sum_{i=1}^ns_i\Z_i\Big]\bigg)
       +\E[s_{n+1}]\,\v^\top\!\adj\!\bigg(\E\Big[\sum_{i=1}^ns_i\Z_i\Big]\bigg)\u
\\\text{(Sylvester's Theorem)}\quad
  &=\det\!\bigg(\E\Big[\sum_{i=1}^ns_i\Z_i\Big]+\E[s_{n+1}]\,\u\v^\top\bigg)\
    =\ \det\!\big(\E[\A]\big),
\end{align*}
showing (a). Finally, (b) follows by applying (a) to each entry  $\adj(\A)_{ij}=(-1)^{i+j}\det(\A_{-j,-i})$.
\end{proof}
Similar expectation identities for the determinant have been given
before \cite{expected-generalized-variance,correcting-bias,dpp-intermediate}.
None of them, however, apply to the random matrix $\A$ as defined in
Lemma \ref{t:det}, or even to the special case we use for analyzing
distributed Newton's method. Also, our proof method is quite
different, and somewhat simpler, than those used in prior work. 
To our knowledge, the extension of determinantal expectation to the adjugate matrix has not previously been pointed out. 

We next prove Theorem \ref{t:key} and Corollary \ref{t:limit} as consequences of Lemma~\ref{t:det}.

\begin{proofof}{Theorem}{\ref{t:key}}
When $\A$ is invertible, its adjugate is given by $\adj(\A)=\det(\A)\A^{-1}$, so the lemma implies that
\begin{align*}
\E\big[\!\det(\A)\big]\big(\E[\A]\big)^{-1} =
  \det\!\big(\E[\A]\big)\big(\E[\A]\big)^{-1}=\adj(\E[\A])
  = \E\big[\!\adj(\A)\big] = \E\big[\!\det(\A)\A^{-1}\big],
\end{align*}
from which Theorem \ref{t:key} follows immediately.
\end{proofof}
\begin{proofof}{Corollary}{\ref{t:limit}}
The subsampled Hessian matrix used in Corollary
\ref{t:limit} can be written as:
  \begin{align*}
\nabla^2\!\Lch(\w) = \frac1k\sum_i b_i\ell_i''(\w^\top\x_i)\,
  \x_i\x_i^\top\, +\lambda\sum_{i=1}^d\e_i\e_i^\top\,\defeq\, \Hbh,
  \end{align*}
so, letting $\Hbh_t =\nabla^2\!\Lch_t(\w)$, Corollary \ref{t:limit} follows from Theorem \ref{t:key}
and the law of large numbers:
\begin{align*}
  \frac{\sum_{t=1}^ma_t\, \pbh_t
  }{\sum_{t=1}^ma_t} &= 
  \frac{\frac1m\sum_{t=1}^m \det\!\big(\Hbh_t\big)
\Hbh_t^{-1}\nabla\!\Lc(\w)}
{\frac1m\sum_{t=1}^m \det\!\big(\Hbh_t\big)}% \nonumber
     \underset{m\rightarrow\infty}{\longrightarrow}
\frac{\E\big[\!\det(\Hbh)\Hbh^{-1}\big]}
{\E\big[\!\det(\Hbh)\big]}\nabla\!\Lc(\w)
 =\nabla^{-2}\!\Lc(\w)\,\nabla\!\Lc(\w),
\end{align*}
which concludes the proof.
 \end{proofof}

\section{Finite-sample convergence analysis}
\label{s:error}
In this section, we prove Theorem \ref{t:finite} and Corollary \ref{t:error}, establishing that 
%under some assumptions 
determinantal averaging exhibits a $1/\sqrt{m}$ convergence rate, where $m$ is the number of sampled matrices (or the number of machines in distributed Newton's method). 
For this, we need a tool from random matrix theory.
\label{s:error}
\begin{lemma}[Matrix Bernstein \cite{matrix-tail-bounds}]\label{t:bernstein}
  Consider a finite sequence $\{\X_i\}$ of independent, random,
  self-adjoint matrices with dimension d such that   $\E[\X_i] =
  \zero$ and $\lambda_{\max}(\X_i)\leq R$ almost surely. If
the sequence satisfies $\big\|\!\sum_i\E[\X_i^2]\big\|\leq \sigma^2$, then
  the following inequality holds for all $x\geq 0$:
  \begin{align*}
    \Pr\!\bigg(\lambda_{\max}\Big(\sum\nolimits_i\X_i\Big) \geq x\bigg) \leq
    \begin{cases}
      d\,\ee^{-\frac{x^2}{4\sigma^2}}
      &\text{for }x\leq \frac{\sigma^2}{R};
      \\       d\,\ee^{-\frac{x}{4R}}
      &\text{for }x\geq \frac{\sigma^2}{R}.
    \end{cases}
  \end{align*}
\end{lemma}
The key component of our analysis is bounding the moments of the
determinant and adjugate of a certain class of random matrices. This
has to be done carefully, because higher moments of the determinant
grow more rapidly than, e.g., for a sub-gaussian random
variable. For this result, we do not require that the individual
components $\Z_i$ of matrix $\A$ be rank-1, but we impose several
additional boundedness assumptions. In the proof below we apply the concentration inequality
of Lemma \ref{t:bernstein} twice: first to the random matrix $\A$
itself, and then also to its trace, which allows finer control over
the determinant.
\begin{lemma}\label{t:moments}
Let $\A=\frac1\gamma\sum_ib_i\Z_i\,+\B$, where
$b_i\sim\mathrm{Bernoulli}(\gamma)$ are independent, whereas $\Z_i$ and $\B$
are $d\times d$ psd matrices such that $\|\Z_i\|\leq \epsilon$ for all $i$ and $\E[\A]=\I$. If
$\gamma\geq 8\epsilon d\eta^{-2}(p+\ln d)$ for $0<\eta\leq0.25$ and
$p\geq 2$, then 
\begin{align*}
\mathrm{(a)}\ \    \E\Big[\big|\det(\A)-1\big|^p\Big]^{\frac1p}\leq 5\eta\quad
  \text{and}\quad
\mathrm{(b)}\ \    \E\Big[\big\|\adj(\A)-\I\big\|^p\Big]^{\frac1p}\leq 9\eta.
\end{align*}
\end{lemma}
\begin{proof}
We start by proving (a). Let $X=\det(\A)-1$ and denote $\one_{[a,b]}$ as the indicator variable of the event that
  $X\in[a,b]$. Since $\det(\A)\geq 0$, we have: 
  \begin{align}
    \E\big[|X|^p\big]
    &  = \E\big[(-X)^p\cdot\one_{[-1,0]}\big]
      + \E\big[X^p\cdot\one_{[0,\infty]}\big]\nonumber
    \\ & \leq \eta^p\ +\
\int_{\eta}^{1} px^{p-1}\Pr(-X\geq x)dx
\ +\ \int_0^\infty px^{p-1}\Pr(X\geq x)dx.\label{eq:integrals}
  \end{align}
 Thus it suffices to bound the two integrals. We will start with the
 first one. Let $\X_i=(1-\frac{b_i}\gamma)\Z_i$. We use the matrix Bernstein
inequality to control the extreme eigenvalues of the matrix
$\I-\A = \sum_i\X_i$ (note that matrix $\B$ cancels out because
$\I=\E[\A]=\sum_i \Z_i+\B$). To do this, observe that $\|\X_i\|\leq 
\epsilon/\gamma$ and, moreover, $\E\big[(1-\frac {b_i}\gamma)^2\big] =
\frac1\gamma-1\leq \frac1\gamma$, so:
\begin{align*}
  \Big\| \sum_i\E[\X_i^2]\Big\|
  &= \Big\|\sum_i\E\big[(1-\tfrac
  {b_i}\gamma)^2\big]\Z_i^2\Big\|
\leq \frac1\gamma \cdot \Big\|\sum_i\Z_i^2\Big\|\leq
    \frac\epsilon\gamma\cdot\Big\|\sum_i\Z_i\Big\|\leq
    \frac\epsilon\gamma. 
\end{align*}
Thus, applying Lemma \ref{t:bernstein} we
conclude that for any $z\in\big[\frac\eta {\sqrt{2d}},1\big]$:
 \begin{align}
   \Pr\!\Big(\|\I-\A\|\geq z\Big)\leq
   2d\,\ee^{-\frac{z^2\gamma}{4\epsilon}}
\leq 2\ee^{\ln(d) - z^2\frac{2d}{\eta^2}(p+\ln d)}\leq
2\ee^{-z^2\frac{2dp}{\eta^2}}.\label{eq:mtail}
 \end{align}
Conditioning on the high-probability event given by \eqref{eq:mtail}
 leads to the lower bound $\det(\A)\geq (1-z)^d$ which is very
 loose. To improve on it, we use the following inequality, where
 $\delta_1,\dots,\delta_d$ denote the eigenvalues of $\I-\A$:
 \begin{align*}
   \det(\A)\ee^{\tr(\I-\A)}=\prod_i(1-\delta_i)\ee^{\delta_i}\geq
   \prod_i(1-\delta_i)(1+\delta_i)=\prod_i(1-\delta_i^2).
 \end{align*}
 Thus we obtain a tighter bound when $\det(\A)$ is multiplied by
 $\ee^{\tr(\I-\A)}$, and now it suffices to upper bound the
 latter. This is a simple application of the scalar Bernstein's
 inequality (Lemma \ref{t:bernstein} with $d=1$)
for the random variables $X_i=\tr(\X_i)\leq
 \epsilon/\gamma\leq \frac{\eta^2}{8dp}$, which satisfy $\sum_i\E[X_i^2]\leq
 \frac\epsilon\gamma\, \tr\big(\sum_i\Z_i\big)\leq \frac{\epsilon d}\gamma\leq\frac{\eta^2}{8p}$.
Thus the scalar Bernstein's inequality states that
 \begin{align}
    \max\Big\{\Pr\!\big(\tr(\A-\I)\geq y\big),\
   \ \Pr\!\big(\tr(\A-\I)\leq -y\big)\Big\}\ \leq\ 
      \begin{cases}
\ee^{-y^2\frac{2p}{\eta^2}}
      &\text{for }y\leq d;
      \\       \ee^{-y\frac{2dp}{\eta^2}}
      &\text{for }y\geq d.
    \end{cases}\label{eq:stail-det}
 \end{align}
 Setting $y=\frac x2$ and $z=\sqrt{\!\frac x{2d}}$ and
 taking a union bound over the appropriate high-probability events
 given by \eqref{eq:mtail} and \eqref{eq:stail-det}, we conclude that
 for any $x\in  [\eta,1]$: 
 \begin{align*}
   \det(\A)\geq (1-z^2)^d\exp\big(\tr(\A-\I)\big)\geq
   \big(1-\tfrac x2\big)\mathrm{e}^{-\frac x 2}\geq 1-x,\quad\text{with
   prob. } 1-3\ee^{-x^2\frac{p}{2\eta^2}}.
 \end{align*}
 Thus, for $X=\det(\A)-1$ and $x\in[\eta,1]$ we obtain that $\Pr(-X\geq x)\leq
 3\ee^{-x^2\!\frac{ p}{2\eta^2}}$, and consequently,
 \begin{align*}
   \int_{\eta}^1\!\!p x^{p-1}\,\Pr\!\big(-X\geq
   x\big) dx
   &\ \leq\ 3p \int_\eta^1\!\!
     x^{p-1}\,\ee^{-x^2\!\frac{ p}{2\eta^2}} dx
\leq 3p\,\sqrt{\pi\tfrac{2\eta^2}p}\cdot \int_{-\infty}^{\infty}\!\!
     |x|^{p-1}\,\frac{\ee^{-x^2\!\frac{ p}{2\eta^2}}}{\sqrt{2\pi\eta^2/p}}
     dx\\[-1mm]
&\leq 3\sqrt{2\pi\eta^2p}\cdot \big(\tfrac{\eta^2}{p}\, p\big)^{\frac
     {p-1}2}\ =\ 3\sqrt{2\pi p}\cdot  \eta^p.
 \end{align*}
 We now move on to bounding the remaining integral from
 \eqref{eq:integrals}. 
 % Let $\lambda_i=\tr(\Z_i)$ and let 
 % $R=\sum_i\lambda_i\leq d$
Since determinant is the product of eigenvalues, we have 
$  \det(\A) = \det(\I+\A-\I)\leq \ee^{\tr(\A-\I)}$, so we can use the
 Bernstein bound of \eqref{eq:stail-det} w.r.t.~$\A-\I$. It follows
 that:
\begin{align*}
  \int_0^\infty\!\! px^{p-1}\Pr(X\geq x)dx
  &\leq \int_0^\infty\!\! px^{p-1}\Pr\big(\ee^{\tr(\A-\I)}\geq 1+x\big)dx
  \\ & \leq \int_0^{\ee^d-1}
     \!\!px^{p-1}\ee^{-\ln^2(1+x)\frac {2p}{\eta^2}}dx\ +\
     \int_{\ee^d-1}^\infty \!\!px^{p-1}\ee^{-\ln(1+x)\frac {2dp}{\eta^2}}dx
  \\ & \leq \int_0^{\ee-1}
     \!\!px^{p-1}\ee^{-\ln^2(1+x)\frac {2p}{\eta^2}}dx\ +\
     \int_{\ee-1}^\infty \!\!px^{p-1}\ee^{-\ln(1+x)\frac {2p}{\eta^2}}dx,
\end{align*}
because $\ln^2(1+x)\geq \ln(1+x)$ for $x\geq \ee-1$.
Note that $\ln^2(1+x)\geq
x^2/4$ for $x\in[0,\ee-1]$, so
\begin{align*}
  \int_0^{\ee-1} \!\!px^{p-1}\ee^{-\ln^2(1+x)\frac {2p}{\eta^2}}dx
  &\leq   \int_0^{\ee-1} \!\!px^{p-1}\ee^{-x^2\!\frac{p}{2\eta^2}}dx
\leq \sqrt{2\pi p}\cdot\eta^p.
\end{align*}
In the interval $x\in[\ee-1,\infty]$, we have:
\begin{align*}
  \int_{\ee-1}^\infty \!\! px^{p-1}\ee^{-\ln(1+x)\frac {2p}{\eta^2}}dx
  &=p\int_{\ee-1}^\infty \!\! \ee^{(p-1)\ln(x)-\ln(1+x)\frac {2p}{\eta^2}}dx
\leq p\int_{\ee-1}^\infty \!\! \ee^{-\ln(1+x)\frac {p}{\eta^2}}dx\\
&       \leq p\int_1^\infty \!\! \big(\tfrac 1{1+x}\big)^{\frac {p}{\eta^2}}dx
       =      \frac{p}{\frac {p}{\eta^2}-1}\big(\tfrac12\big)^{\frac {p}{\eta^2}-1}
\leq p\cdot\Big(\big(\tfrac12\big)^{\frac1{\eta^2}}\Big)^{p}\
     \leq\ p\cdot\big(\eta^2\big)^{p},
  \end{align*}
where the last inequality follows because
$(\frac12)^{\frac1{\eta^2}}\leq \eta^2$. Noting that
$(1+4\sqrt{2\pi p}+p)^{\frac1p}\leq 5$ for any $p\geq 2$ concludes the
proof of (a). The proof of (b), given in Appendix \ref{a:moments},
follows similarly as above because for any positive
definite $\A$ we have $\frac{\det(\A)}{\lambda_{\max}(\A)}\cdot\I\preceq
\adj(\A)\preceq \frac{\det(\A)}{\lambda_{\min}}\cdot \I$.
\end{proof}
Having obtained bounds on the higher moments, we can now convert them to
convergence with high probability for the average of determinants and
the adjugates. Since determinant is a scalar variable, this follows
by using standard arguments. On the other hand, for the adjugate
matrix we require a somewhat less standard matrix extension of the
Khintchine/Rosenthal inequalities (see Appendix~\ref{a:moments}).
\begin{corollary}\label{c:moments}
There is $C>0$ s.t.~for $\A$ as in Lemma \ref{t:moments}
with all $\Z_i$ rank-$1$ and $\gamma\geq C\epsilon
d\eta^{-2}\log^3\!\frac d\delta$, 
  \begin{align*}
    (a)\ \Pr\!\bigg(\Big|\frac1m\sum_{t=1}^m\det(\A_t) - 1\Big|\geq
\!\frac\eta{\sqrt{m}}\bigg)\leq\delta\quad\text{and}\quad
   (b)\   \Pr\!\bigg(\Big\|\frac1m\sum_{t=1}^m\adj(\A_t) - \I\Big\|\geq
    \!\frac\eta{\sqrt{m}}\bigg)\leq\delta,
  \end{align*}
  where $\A_1,\dots,\A_m$ are independent copies of $\A$.
\end{corollary}
We are ready to show the convergence rate of determinantal averaging,
which follows essentially by upper/lower bounding the enumerator
and denominator separately, using Corollary \ref{c:moments}.
\begin{proofof}{Theorem}{\ref{t:finite}}
We will apply Corollary \ref{c:moments} to the matrices
$\A_t=\H^{-\frac12}\Hbh_t\H^{-\frac12}$. Note that
$\A_t=\frac nk\sum_ib_i\Zbt_i+\lambda\H^{-1}$, where
each $\Zbt_i=\frac1n\H^{-\frac12}\Z_i
\H^{-\frac12}$ satisfies $\|\Zbt_i\|\leq \mu\cdot d/n$. Therefore, Corollary
\ref{c:moments} guarantees that for $\frac kn\geq C\frac{\mu
  d}{n}d\eta^{-2}\log^3\!\frac d\delta$, with probability $1-\delta$
the following average of determinants is concentrated around 1:
\begin{align*}
  Z \ \defeq\ \frac1m\sum_t\frac{\det(\Hbh_t)}{\det(\H)}\ =\ 
\frac1m\sum_t\det\!\big(\H^{-\frac12}\Hbh_t\H^{-\frac12}\big)\ \in\
  [1-\alpha,1+\alpha]\quad\text{ for }\alpha=\frac\eta{\sqrt{m}},
  \end{align*}
  along with a corresponding bound for the adjugate matrices.
  We obtain that with probability $1-2\delta$,
  \begin{align*}
  \bigg\|\,\frac{\sum_{t=1}^m\adj(\A_t)}
    {\sum_{t=1}^m\det(\A_t)} - \I\,\bigg\|
    &\leq \Big\|\frac1m\sum_t\adj\!
     \big(\A_t\big)-Z\,\I\Big\|\,/Z
\\     \text{(Corollary \ref{c:moments}a)}\quad
&\leq\frac1{1-\alpha}\Big\|\frac1m\sum_t\adj\!
     \big(\A_t\big)
     -\I\Big\|+\frac\alpha{1-\alpha}
   \\\text{(Corollary \ref{c:moments}b)}\quad
&\leq\frac\alpha{1-\alpha} + \frac\alpha{1-\alpha}.
  \end{align*}
It remains to multiply the above expressions by $\H^{-\frac12}$ from
both sides to recover the desired estimator:
\begin{align*}
\frac{\sum_{t=1}^m\det(\Hbh_t)\,\Hbh_t^{-1}}
    {\sum_{t=1}^m\det(\Hbh_t)} = \H^{-\frac12}\,\frac{\sum_{t=1}^m\adj(\A_t)}
    {\sum_{t=1}^m\det(\A_t)}\,\H^{-\frac12}\preceq
  \H^{-\frac12}
  \big(1+\tfrac{2\alpha}{1-\alpha}\big)\,\I\,\H^{-\frac12}
  =\big(1+\tfrac{2\alpha}{1-\alpha}\big)\H^{-1},
\end{align*}
and the lower bound follows identically. Appropriately adjusting the
constants concludes the proof.
\end{proofof}

As an application of the above result, we show how this allows us to
bound the estimation error in distributed Newton's method, when
using determinantal averaging.
\begin{proofof}{Corollary}{\ref{t:error}}
Follows from Theorem \ref{t:finite} by setting
$\Z_i=\ell_i''(\w^\top\x_i)\x_i\x_i^\top$ and $\B=\lambda\I$. Note
that the assumptions imply that $\|\Z_i\|\leq \mu$, so invoking the
theorem and denoting $\g$ as $\nabla\!\Lc(\w)$, with probability $1-\delta$ we have
\begin{align*}
    \bigg\|\,\frac{\sum_{t=1}^ma_t\,\pbh_t}{\sum_{t=1}^ma_t} \,-\,
  \p\,\bigg\|_{\H} &=
\bigg\|  \H^{\frac12}\bigg(\frac{\sum_{t=1}^m\det(\Hbh_t)\,\Hbh_t^{-1}}
    {\sum_{t=1}^m\det(\Hbh_t)} -
                     \H^{-1}\bigg)\H^{\frac12}\,\H^{-\frac12}\g\bigg\|
\\ &\leq \bigg\|  \H^{\frac12}\bigg(\frac{\sum_{t=1}^m\det(\Hbh_t)\,\Hbh_t^{-1}}
    {\sum_{t=1}^m\det(\Hbh_t)} -
                     \H^{-1}\bigg)\H^{\frac12}\bigg\|\cdot
     \big\|\H^{-\frac12}\g\big\|
  \\
  \text{(Theorem \ref{t:finite})}\quad
&\leq
     \big\|\H^{\frac12}\tfrac\eta{\sqrt{m}}\H^{-1}\H^{\frac12}\big\|\cdot\|\p\|_{\H}\
     = \ \tfrac\eta{\sqrt{m}}\cdot\|\p\|_{\H},
\end{align*}
which completes the proof of the corollary.
\end{proofof}

\subsubsection*{Acknowledgements}
MWM would like to acknowledge ARO, DARPA, NSF and ONR for providing partial
  support of this work. Also, MWM and MD thank the NSF for
  funding via the NSF TRIPODS program. Part of this work
  was done while MD and MWM were visiting the Simons Institute for the
  Theory of Computing.

\bibliographystyle{alpha}
\bibliography{pap}

\newpage
\appendix

\section{Omitted proofs from Section \ref{s:error}}
\label{a:moments}

In Section \ref{s:error}, we stated Lemma \ref{t:moments} and proved
the first part of it (a moment bound for the determinant).
Here, we provide the proof of the second part (a moment bound for the
adjugate).

\begin{lemma}[Lemma \ref{t:moments}b restated]
Let $\A=\frac1\gamma\sum_ib_i\Z_i\,+\B$, where
$b_i\sim\mathrm{Bernoulli}(\gamma)$ are independent, whereas $\Z_i$ and $\B$
are $d\times d$ psd matrices such that $\|\Z_i\|\leq \epsilon$ for all $i$ and $\E[\A]=\I$. If
$\gamma\geq 8\epsilon d\eta^{-2}(p+\ln d)$ for $0<\eta\leq0.25$ and
$p\geq 2$, then 
\begin{align*}
  \E\Big[\big\|\adj(\A)-\I\big\|^p\Big]^{\frac1p}\leq 9\eta.
\end{align*}
\end{lemma}
\begin{proof}
  Let $\lambda_{\max}$ and $\lambda_{\min}$ denote the largest and
  smallest eigenvalue of $\adj(\A)$. We have
\begin{align*}
   \E\big[\|\adj(\A)-\I\|^p\big]
    &= \int_0^\infty
    \!\!\! px^{p-1}\Pr\big(\|\adj(\A)-\I\|\geq x\big)dx
\\ & \leq    \eta^p+\int_{\eta}^\infty \!\!\!px^{p-1} \!\Big(\!\Pr(\lambda_{\max}\geq 1+x) +
     \Pr(\lambda_{\min}\leq1-x)\Big)dx.
\end{align*}
We will now bound the two probabilities. Let $\delta_{\max}$ and
$\delta_{\min}$ denote the largest and smallest eigenvalue of
matrix $\A-\I$. Recall
the following concentration bounds implied by Lemma
\ref{t:bernstein} (see the first part of the
proof of Lemma \ref{t:moments}):
 \begin{align}
   \max\Big\{\Pr\!\big(\tr(\A-\I)\geq y\big),\
   \ \Pr\big(\tr(\A-\I)\leq -y\big)\Big\}\
   &\leq\
      \begin{cases}
      \ee^{-y^2\frac{2p}{\eta^2}}
      &\text{for }y\in [0,d];
      \\       \ee^{-y\,\frac{2dp}{\eta^2}}
      &\text{for }y\geq d,
    \end{cases}\label{eq:stail}
   \\
   \max\Big\{\Pr\big(\delta_{\max}\geq z\big),\ \
   \Pr\big(\delta_{\min}\leq -z\big)\Big\}\ 
   &\leq\
     \begin{cases}
              \ee^{-z^2\frac {2p}{\eta^2}}
       &\text{ for }z\in[0,\frac\eta{\sqrt{ 2d}}];
\\   \ee^{-z^2\frac {2dp}{\eta^2}}
       &\text{ for }z\in[\frac\eta{\sqrt{2d}},1];
       \\ \ee^{-z\frac{2dp}{\eta^2}}
       &\text{ for }z\geq 1.
       \end{cases}
 \end{align}
From the formula $\adj(\A)=\det(\A)\A^{-1}$ it follows
that $\lambda_{\max}\leq
\frac{\det(\A)}{1+\delta_{\min}}\leq\frac{\ee^{\tr(\A-\I)}}{1+\delta_{\min}}$
so we have %if $x\in[\eta,1]$ then:
 \begin{align*}
   \Pr\big(\lambda_{\max}\!\geq\! 1+x\big)
   &\leq \Pr\bigg(\frac{\ee^{\tr(\A-\I)}}{1+\delta_{\min}}\geq 1+x\bigg)
\\ & =\Pr\Big(\tr(\A-\I) + \ln\frac1{1+\delta_{\min}}\geq
     \ln(1+x)\Big)
\\ &\leq \Pr\Big(\tr(\A-\I)\geq \frac23\cdot\ln(1+x)\Big)\ +\ 
     \Pr\Big(\ln\frac1{1+\delta_{\min}}\geq \frac13\cdot\ln(1+x)\Big)
\\ &=\Pr\Big(\tr(\A-\I)\geq \frac23\cdot\ln(1+x)\Big)\ +\ 
     \Pr\Big(\delta_{\min}\leq \frac1{(1+x)^{\frac13}} - 1\Big).
   \\ &\leq
     \begin{cases}
\ee^{-\ln^2\!(1+x)\frac {8p}{9\eta^2}} +
   \ee^{-(1-(\frac1{1+x})^{\frac13})^2\frac{2p}{\eta^2}}\leq
   2\ee^{-x^2\frac{p}{20\eta^2}}
   &\text{ for }x\in[0,\ee\!-\!1],
   \\ \ee^{-\ln(1+x)\frac {4p}{3\eta^2}}
   + \ee^{-\frac1{16}\frac {2dp}{\eta^2}}\qquad\quad
   \leq 2\ee^{-\ln(1+x)\frac{p}{8\eta^2}}
&\text{ for }x\in[\ee\!-\!1,\ee^d\!-\!1].
\end{cases}
 \end{align*}
For $x\geq \ee^d-1$, since $\lambda_{\max}\leq
(1+\delta_{\max})^d\leq e^{d\delta_{\max}}$ and $\ln(1+x)\geq d$, we have: 
\begin{align*}
  \Pr\big(\lambda_{\max}\!\geq\! 1+x\big)\leq
  \Pr\big(\ee^{d\delta_{\max}}\!\geq\! 1+x\big)=
  \Pr\big(\delta_{\max}\!\geq\! \ln(1+x)/d\big)\leq
  \ee^{-\ln(1+x)\frac {2p}{\eta^2}}.
\end{align*}
Next, we use the fact that for $\delta=
\max\big\{|\delta_{\max}|,|\delta_{\min}|\big\}$ we have: 
\begin{align*}
  \lambda_{\min}\geq
\frac{\det(\A)}{1+\delta_{\max}}\geq
\frac{(1-\delta^2)^d}{(1+\delta_{\max})\ee^{\tr(\I-\A)}}
  \geq (1-\delta)(1-d\delta^2)\big(1-\tr(\I-\A)\big),
\end{align*}
so for $x\in[\eta,1]$ we have:
\begin{align*}
  \Pr\big(\lambda_{\min}\leq 1-x\big)
  &\leq \Pr(\delta\geq x/3) +
  \Pr(\delta^2\geq x/3d) + \Pr\big(\tr(\I-\A)\geq x/3\big)
  \\ &
\leq 2\ee^{-x^2\frac{2dp}{9\eta^2}}+2\ee^{-\frac
       x{3d}\frac{2dp}{\eta^2}} + \ee^{-x^2\frac {2p}{\eta^2}}\ \leq\
       5\ee^{-x^2\frac {2p}{9\eta^2}}.
\end{align*}
Putting everything together we obtain that:
\begin{align*}
  \E\big[\|\adj(\A)-\I\|^p\big]
&\leq
    \eta^p + \int_\eta^{\ee-1}px^{p-1}\,
  7\ee^{-x^2\frac p{20\eta^2}}\d x 
+ \int_{\ee-1}^\infty px^{p-1}
     3\ee^{-\ln(1+x)\frac {p}{8\eta^2}}\d x
  \\ &
\leq \eta^p + 7\sqrt{20\pi p}\,\eta^p 
       +\frac{3p}{\frac{p}{16\eta^2}-1}\big(\tfrac12\big)^{\frac{p}{16\eta^2}-1}
  \\ &
\leq \eta^p + 7\sqrt{20\pi p}\,\eta^p  + 6(3\eta)^p\ \leq \ (9\eta)^p,
\end{align*}
which completes the proof.
\end{proof}

As a consequence of the moment bounds shown in Lemma
\ref{t:moments}, we establish convergence with high probability for
the average of determinants and 
the adjugates. For the adjugate matrix, we require a matrix variant of the
Khintchine/Rosenthal inequalities.
\begin{lemma}[\cite{cgt12}]\label{t:rosenthal}
  Suppose that $p\geq 2$ %$d\geq 3$
  and $r=\max\{p,2\log d\}$.%
  \footnote{In \cite{cgt12} it is assumed that $d\geq 3$, however this
    assumption is not used anywhere in the proof.} Consider a finite
   sequence $\{\X_i\}$ of independent, symmetrically random, self-adjoint
   matrices with dimension $d\times d$. Then,
   \begin{align*}
     \E\Big[\big\|\sum\nolimits_i\X_i\big\|^p\Big]^{\frac 1p}
     \leq\sqrt{\ee r}\,
     \Big\|\sum\nolimits_i\E[\X_i^2]\Big\|^{\frac 12}+
     2\ee r \,\E\big[\!\max\nolimits_i\|\X_i\|^p\big]^{\frac 1p}.
   \end{align*}
 \end{lemma}
\begin{corollary}[Corollary~\ref{c:moments} restated]
There is $C>0$ s.t.~for $\A$ as in Lemma \ref{t:moments}
with all $\Z_i$ rank-$1$ and $\gamma\geq C\epsilon
d\eta^{-2}\log^3\!\frac d\delta$, 
  \begin{align*}
    (a)\ \Pr\!\bigg(\Big|\frac1m\sum_{t=1}^m\det(\A_t) - 1\Big|\geq
\!\frac\eta{\sqrt{m}}\bigg)\leq\delta\quad\text{and}\quad
   (b)\   \Pr\!\bigg(\Big\|\frac1m\sum_{t=1}^m\adj(\A_t) - \I\Big\|\geq
    \!\frac\eta{\sqrt{m}}\bigg)\leq\delta,
  \end{align*}
  where $\A_1,\dots,\A_m$ are independent copies of $\A$.
\end{corollary}
\begin{proof}
  Applying Lemma \ref{t:moments} to the matrix
$\A$, for appropriate $C$ and any fixed $p\geq 2$, if
$\gamma\geq C\epsilon d\sigma^{-2}(p+\ln d)$, then for any $s\in[2,p]$ we have
$\E\big[\|\adj(\A_t)-\I\|^s\big]\leq \sigma^s$.
With the additional assumption that $\Z_i$'s are rank-$1$,
Theorem \ref{t:det} implies that $\E\big[\adj(\A_t)\big]=\I$, so by
a standard symmetrization argument, where $r_t$ denote independent
Rademacher random variables,
\begin{align*}
  \E\bigg[\Big\|\frac1m\sum_{t=1}^m\adj(\A_t)-\I\Big\|^p\bigg]^{\frac1p}
  &\leq 2\cdot \E\bigg[\Big\|\sum_t
      \frac{r_t}m\big(\adj(\A_t)-\I\big)\Big\|^p\bigg]^{\frac1p}.
\end{align*}
Applying Lemma \ref{t:rosenthal} to the matrices
$\X_t=\frac1m\Y_t$, where $\Y_t=r_t\big(\adj(\A_t)-\I\big)$, we obtain that:
\begin{align*}
  \E\bigg[\Big\|\frac1m\sum_t\Y_t\Big\|^p\bigg]^{\frac1p}
&\leq \sqrt{\ee r}\,
     \Big\|m\cdot\frac1{m^2} \E[\Y_t^2]\Big\|^{\frac 12}+
\frac{2\ee r}m\,\E\Big[\sum_{i=1}^m\|\Y_i\|^p\Big]^{\frac1p}
\\ &\leq\sqrt{\frac{\ee r}m}\cdot\E\big[\|\Y\|^2\big]^{\frac12}+\frac{2\ee r}m
     \Big(m\cdot\E\big[\|\Y\|^p\big]\Big)^{\frac1p}
\\ &\leq \bigg(\sqrt{\frac{\ee r}m} + \frac{2\ee
     r}{m^{1-\frac1p}}\bigg)\cdot \sigma\leq
     C'\cdot\frac{p\sigma}{\sqrt{m}},
\end{align*}
for $p\geq 2\log d$ and $C'$ chosen appropriately. Now Markov's inequality yields:
\begin{align*}
    \Pr\bigg(\Big\|\frac1m\sum_t\adj(\A_t)-\I\Big\|\geq \alpha\bigg)
    &\leq
    \alpha^{-p}\cdot
    \E\bigg[\Big\|\frac1m\sum_t
      \adj(\A_t)-\I\Big\|^p\bigg]
\leq \bigg(\frac{2C'p\sigma}{\alpha\sqrt{m}}\bigg)^p.
\end{align*}
Setting $\alpha\!=\!\frac{\eta}{\sqrt m}$, $\sigma\!=\!\frac{\eta}{4C'p}$ and
$p=2\,\lceil\max\{\log d, \log\frac1\delta\}\rceil$, the above bound becomes
$(\frac12)^{p}\leq\delta$ for
$k\geq C''\mu d^2\eta^{-2}(\log^3\!\frac1\delta+\log^3\! d)$.
Showing the analogous result for the average of determinants of matrices $\A_t$ instead of
the adjugates follows identically, except that Lemma
\ref{t:rosenthal} can be replaced with the standard scalar Rosenthal's inequality.
\end{proof}

\section{Proof of Newton convergence}
\label{a:newton}

Here, we provide a proof of Corollary \ref{c:rate}, which describes
the convergence guarantees for the approximate Newton step obtained via
determinantal averaging. It suffices to show the following lemma.

\begin{lemma}\label{l:rate}
Let loss $\Lc$ be defined as in \eqref{eq:loss} and assume its Hessian
is L-Lipschitz (Assumption \ref{a:lipschitz}). If
\begin{align*}
  \big\|\pbh - \p^*\big\|_{\nabla^2\!\Lc(\w)}\leq \alpha\,
  \|\p^*\|_{\nabla^2\!\Lc(\w)},\quad\text{where}\quad \p^* = \nabla^{-2}\!\Lc(\w)\,\nabla\!\Lc(\w),
\end{align*}
then the approximate Newton step $\wbt = \w - \pbh$ satisfies:
\begin{align*}
  \|\wbt-\w^*\|\leq \max\Big\{\alpha\sqrt\kappa\, \|\w-\w^*\|,\
  \frac{2L}{\sigma_{\min}}\,\|\w-\w^*\|^2\Big\},\quad\text{where }\w^*=\argmin_\w\Lc(\w),
\end{align*}
where $\kappa$ and $\sigma_{\min}$ are the condition number
and smallest eiganvalue of $\nabla^{2}\!\Lc(\w)$, respectively.
\end{lemma}
\begin{proof}
The lemma essentially follows via the standard analysis of the
Newton's method. For the sake of completeness we will outline the
proof following \cite{distributed-newton}. Denoting $\H=\nabla^2\!\Lc(\w)$ and
$\g=\nabla\!\Lc(\w)$, we define the auxiliary function
\begin{align*}
  \phi(\p) \defeq \p^\top\H\p - 2\p^\top\g.
\end{align*}
By definition of $\phi(\p)$ we have $\phi(\p^*) =\phi(\H^{-1}\g)=
-\|\p^*\|_{\H}^2$. If follows that
\begin{align*}
  \phi(\pbh)-\phi(\p^*)
  &= \|\H^{\frac12}\pbh\|^2-2\g^\top\H^{-1}\H\pbh + \|\H^{\frac12}\p^*\|^2
  \\ & = \|\H^{\frac12}(\pbh-\p^*)\|^2
= \big\|\pbh-\p^*\big\|_{\H}^2\leq \alpha^2\,\|\p^*\|_{\H}^2 = -\alpha^2\phi(\p^*).
\end{align*}
We invoke the classical result in local convergence analysis of Newton's method
\cite{nocedal-wright}, using the statement of Lemma 9 in \cite{distributed-newton}.
\begin{lemma}[\cite{distributed-newton}]\label{l:9}
  Assume Hessian is L-Lipschitz and that $\pbh$ satisfies
  $\phi(\pbh)\leq(1-\alpha^2)\,\min_\p\phi(\p)$. Then $\wbt=\w-\pbh$ satisfies
  \begin{align*}
    \|\wbt-\w^*\|_{\H}^2\leq L\,\|\w-\w^*\|^2\|\wbt-\w^*\| + \frac{\alpha^2}{1-\alpha^2}\|\w-\w^*\|_{\H}^2.
  \end{align*}
\end{lemma}
Lemma \ref{l:9} immediately implies that one of the following two
inequalities hold:
\begin{align*}
  \|\wbt-\w^*\|&\leq \frac{2L}{\sigma_{\min}(\H)}\cdot\|\w-\w^*\|^2,
\\ \|\wbt-\w^*\|&\leq
                  \frac{\alpha}{\sqrt{1-\alpha^2}}\sqrt{\frac{2\lambda_{\max}(\H)}{\lambda_{\min}(\H)}}\cdot
                  \|\w-\w^*\|,
\end{align*}
which proves Lemma \ref{l:rate}.
\end{proof}
Note that Corollary \ref{c:rate} follows immediately by combining
Corollary \ref{t:error} with Lemma \ref{l:rate}.

\begin{figure}[H]
\includegraphics[width=0.5\textwidth]{figs/vol-unif-abalone}\nobreak
\includegraphics[width=0.5\textwidth]{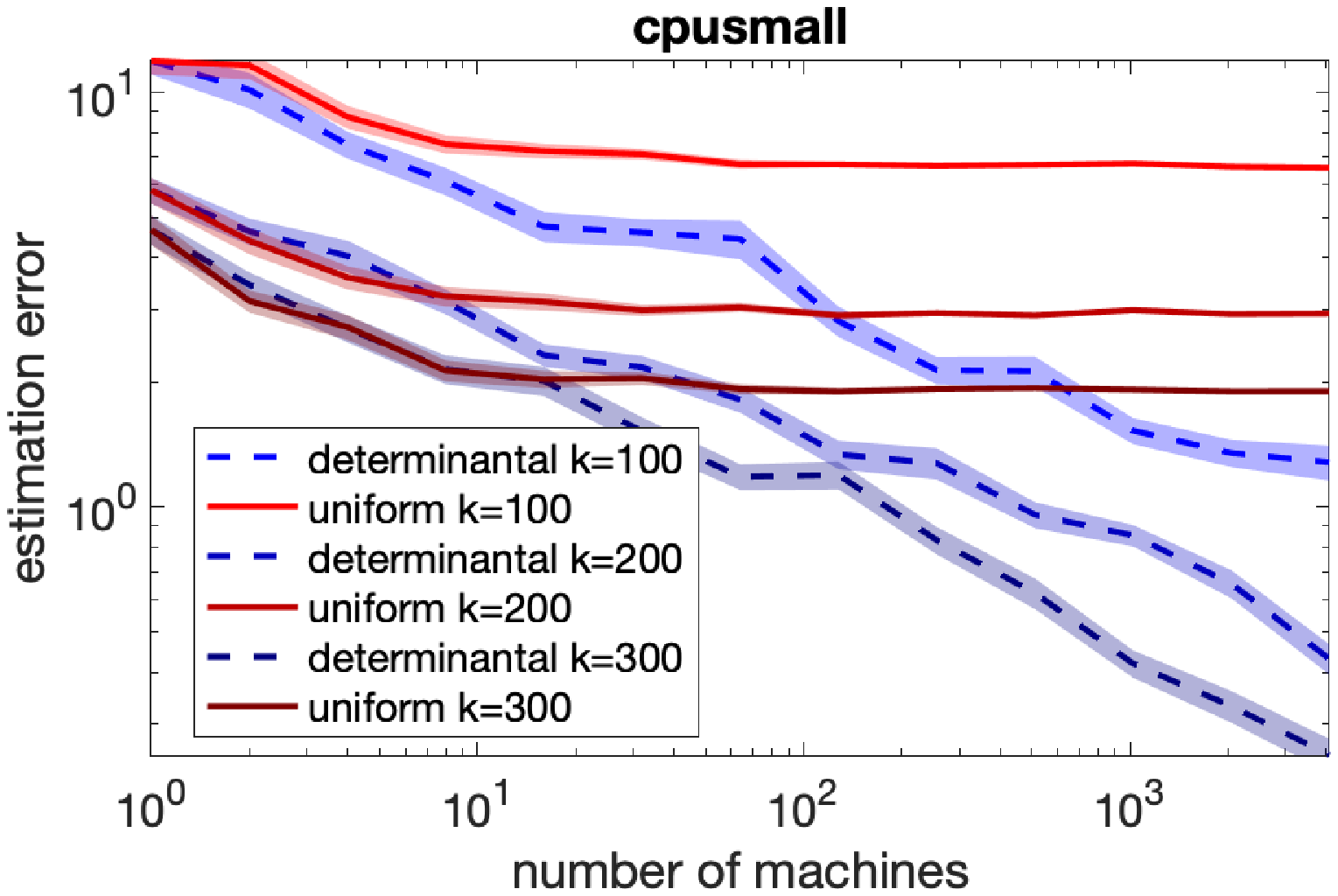}
\includegraphics[width=0.5\textwidth]{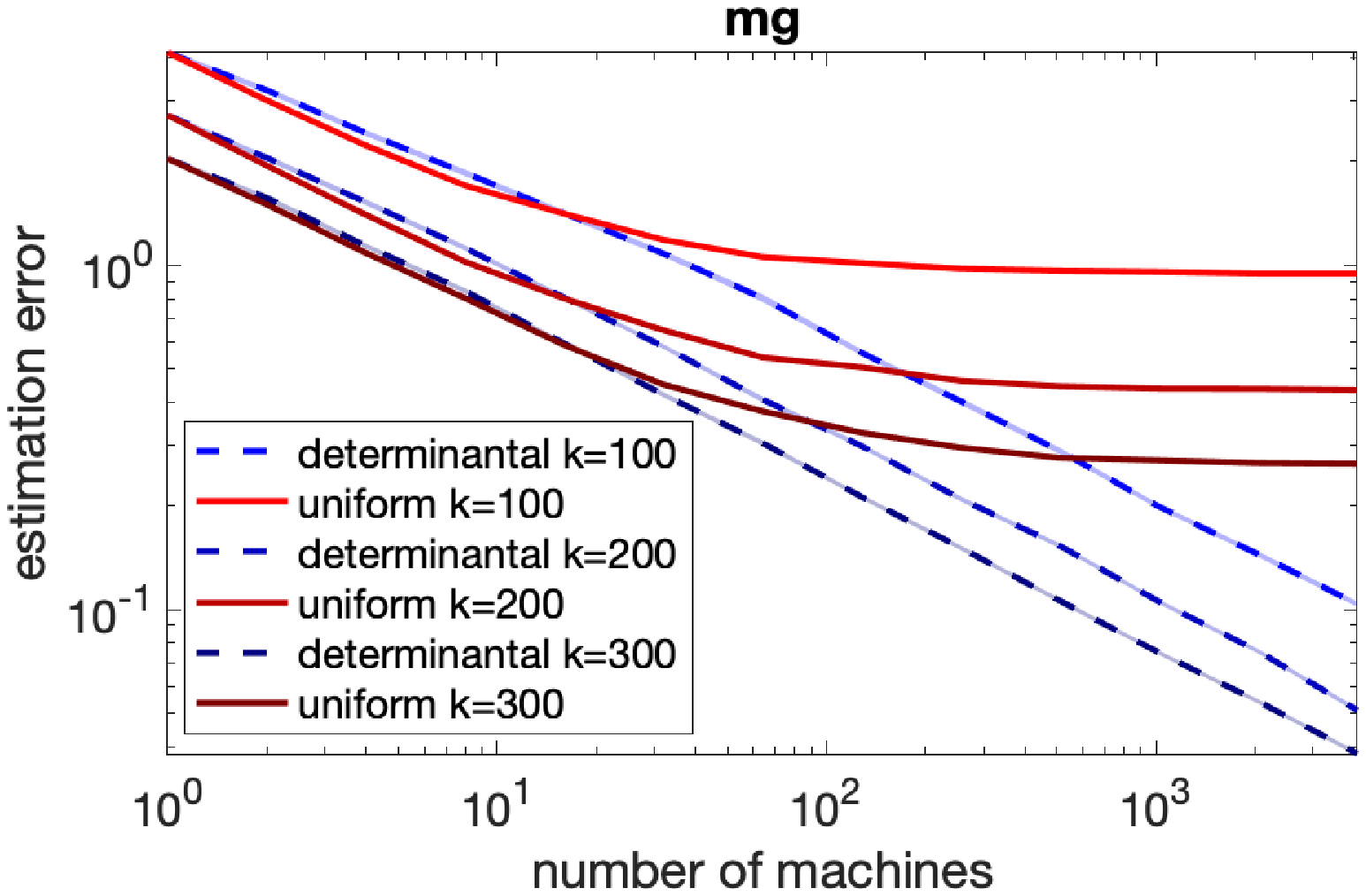}\nobreak
\includegraphics[width=0.5\textwidth]{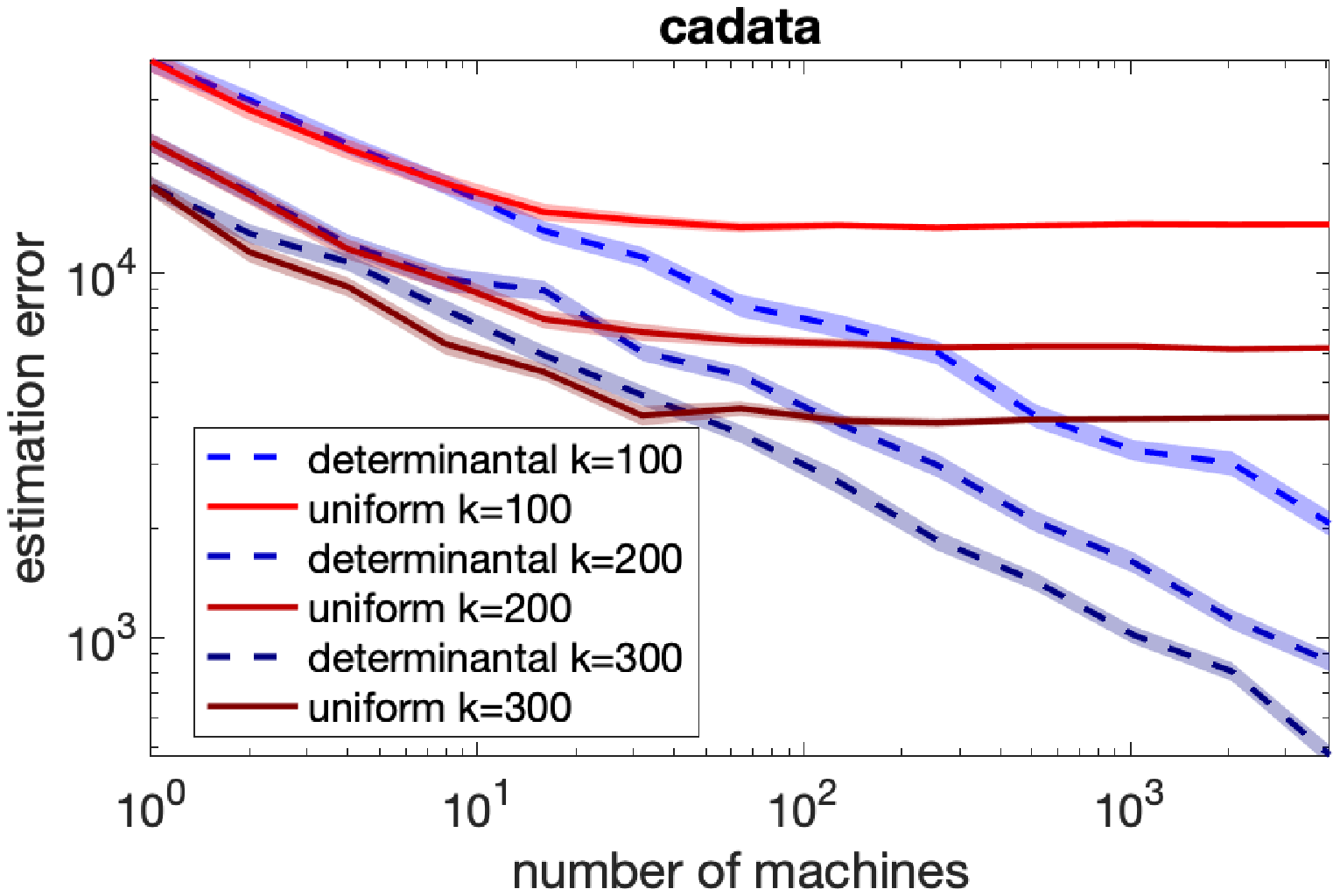}
\caption{Comparison of the estimation error between
\emph{determinantal} and \emph{uniform} averaging on four libsvm datasets.}
\label{f:plots}
\end{figure}

\section{Experiments}
\label{a:experiments}
In this section, we experimentally evaluate the estimation error of
determinantal averaging for the Newton's method (following the setup of
Section \ref{ss:newton}), and we compare it against
uniform averaging \cite{distributed-newton}. Although there are
obvious follow-up directions for empirical and implementational work,
here we focus our experiments on demonstrating the existence
of the inversion bias with previous methods and how our determinantal
averaging solves this problem. We use square loss
$\ell_i(\w^\top\x_i) = (\w^\top\x_i-y_i)^2$, where $y_i$ are the
real-valued labels for a regression problem, and we run the experiments on
several benchmark regression datasets from the libsvm repository \cite{libsvm}. In this setting,
the local Newton estimate computed from the starting vector $\w=\zero$
is given by: 
\begin{align*}
  \pbh =
  \bigg(\frac1k\sum_{i=1}^nb_i\x_i\x_i^\top+\lambda\I\bigg)^{-1}\frac1n\sum_{i=1}^ny_i\x_i,\quad\text{where}\quad b_i\sim\mathrm{Bernoulli}(k/n).
\end{align*}
In all of our experiments we set the regularization parameter to
$\lambda=\frac1n$. Let $\pbh_1,\dots,\pbh_m\simiid \pbh$ be $m$ distributed
local estimates and denote $\Hbh_t$ as the $t$th local Hessian
estimate. The two averaging strategies we compare are:  
\begin{align*}
  \text{determinantal:}\quad \pbh_{\det} =
  \frac{\sum_{t=1}^m\det(\Hbh_t)\,\pbh_t}{\sum_{t=1}^m\det(\Hbh_t)},
  \qquad\text{uniform:}\quad \pbh_{\mathrm{uni}}
  =\frac1m\sum_{t=1}^m\pbh_t.
\end{align*}
Figure \ref{f:plots} plots the estimation errors
$\|\pbh_{\det}-\p^*\|$ and $\|\pbh_{\mathrm{uni}}-\p^*\|$, where
$\p^*$ is the exact Newton step starting from $\w=\zero$, for
datasets \textsc{abalone}, \textsc{cpusmall}, \textsc{mg}%
\footnote{We expanded features to all degree 2 monomials, and removed redundant ones.}
and \textsc{cadata} \cite{libsvm} (for convenience, the plot from
Figure \ref{fig:abalone} in Section \ref{ss:newton} is repeated here). The reported results are averaged
over 100 trials, with shading representing standard error. We
consistently observe that for a small number of machines $m$ both methods effectively
reduce the estimation error, however after a certain point uniform averaging converges to a biased
estimate and the estimation error flattens out. On the other hand,
determinantal averaging continues to converge to the optimum
as the number of machines keeps growing. We remark that for some datasets
determinantal averaging exhibits larger variance than uniform
averaging, especially when local sample size is small. Reducing that
variance, for example through some form of additional regularization,
is a new direction for future work.

\end{document}